\documentclass{article}

\usepackage[preprint]{template/neurips_2026}

\usepackage[utf8]{inputenc} 
\usepackage[T1]{fontenc}    
\usepackage{hyperref}       
\usepackage{url}            
\usepackage{booktabs}       
\usepackage{amsfonts}       
\usepackage{nicefrac}       
\usepackage{microtype}      
\usepackage{xcolor}         

\usepackage{caption}
\usepackage{graphicx}
\usepackage{multirow}
\usepackage[table,xcdraw]{xcolor}
\usepackage{array}
\usepackage{booktabs}
\usepackage{enumitem}
\usepackage{makecell}
\usepackage{wrapfig}
\usepackage{tikz}
\definecolor{oliveGreen}{HTML}{7A9A8D}
\definecolor{coptMainBg}{HTML}{EEF3EF}
\definecolor{brickRed}{HTML}{B55245}
\definecolor{defaultBg}{HTML}{EAF2FB}
\definecolor{coptAuxBg}{HTML}{F7F3EC}
\definecolor{DeepSlate}{HTML}{2F3E46}
\definecolor{MutedTeal}{HTML}{356859}
\definecolor{Auburn}{HTML}{7A3E2D}
\definecolor{DarkSteelBlue}{HTML}{2B4C7E}
\hypersetup{
  colorlinks=true,
  linkcolor=DeepSlate,
  anchorcolor=DeepSlate,
  citecolor=Auburn,
  urlcolor=DarkSteelBlue,
}
\usepackage{pifont}

\usepackage{algorithm}
\usepackage[noend]{algpseudocode} 
\algrenewcommand\algorithmicrequire{\textbf{Input:}}
\algrenewcommand\algorithmicensure{\textbf{Output:}}
\algrenewcommand\textproc{\textsc} 
\algrenewcommand{\algorithmiccomment}[1]{\hfill\textcolor{oliveGreen}{// \textit{#1}}}
\algrenewcommand\alglinenumber[1]{\footnotesize #1}

\usepackage{xspace}
\usepackage{amsmath,amssymb}
\usepackage{amsthm}
\usepackage{bbm}
\usepackage[most]{tcolorbox} 
\usepackage{cleveref}

\newcommand{\copt}{\textsc{CopT}\xspace}

\newtheorem{theorem}{Theorem}
\newtheorem{assumption}{Assumption}

\usepackage{textcomp}  
\usepackage{scalerel}  
\def\hf{\scalerel*{\includegraphics{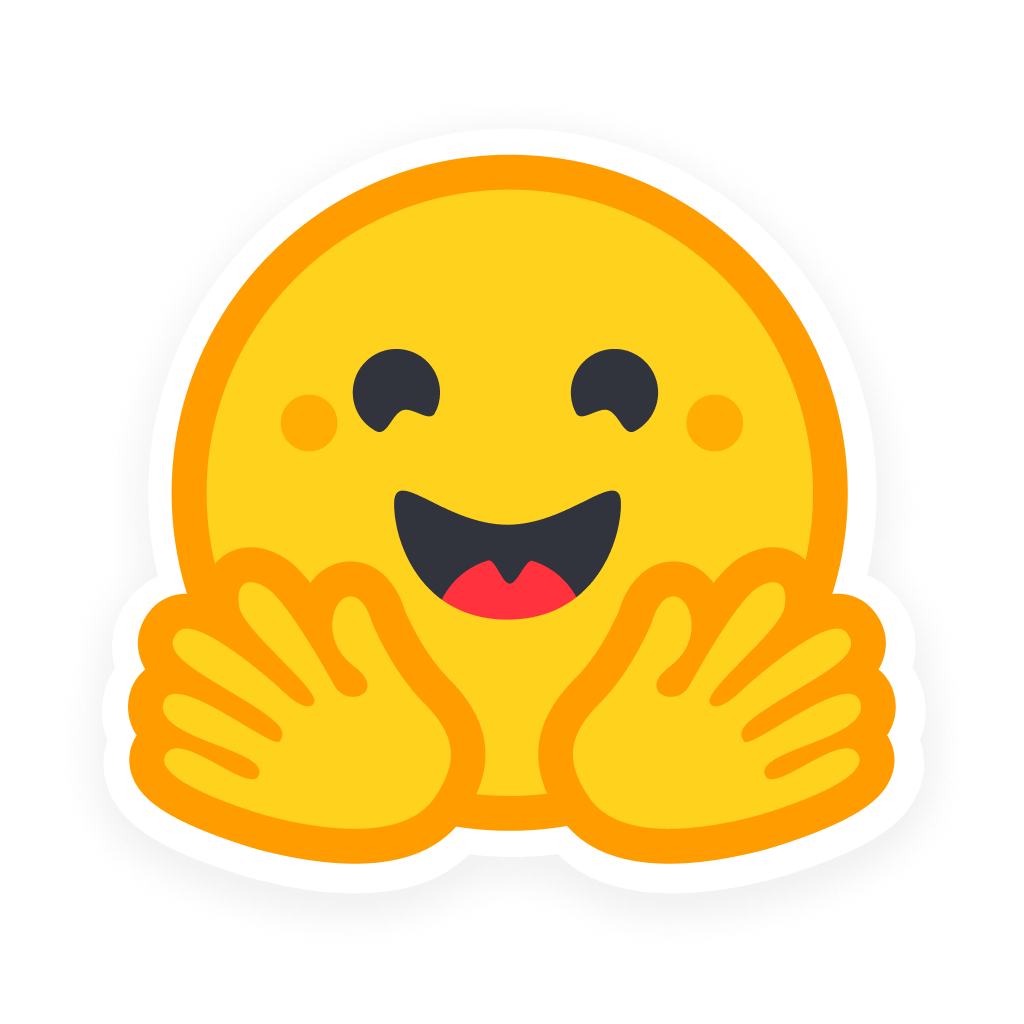}}{\textrm{\textbigcircle}}}
\def\github{\scalerel*{\includegraphics{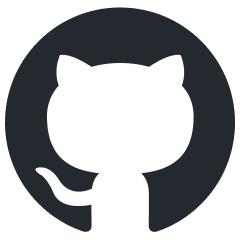}}{\textrm{\textbigcircle}}}

\newcommand{\deltasize}{\fontsize{6.5pt}{6.5pt}\selectfont}

\newcommand{\accup}[2]{#1{\deltasize\textcolor{oliveGreen}{~(+#2)}}}
\newcommand{\accdown}[2]{#1{\deltasize\textcolor{blue!70!black}{~(-#2)}}}
\newcommand{\tokup}[2]{#1{\deltasize\textcolor{oliveGreen}{~(+#2\%)}}}
\newcommand{\tokdown}[2]{#1{\deltasize\textcolor{blue!70!black}{~(#2\%)}}}
\newcommand{\accupbr}[2]{%
  \shortstack[c]{#1\\[-1pt]{\scriptsize\textcolor{oliveGreen}{(+#2)}}}%
}

\newcommand{\tokdownbr}[2]{%
  \shortstack[c]{#1\\[-1pt]{\scriptsize\textcolor{blue!70!black}{(#2\%)}}}%
}
\newcommand{\tokupbr}[2]{%
  \shortstack[c]{#1\\[-1pt]{\scriptsize\textcolor{oliveGreen}{(+#2\%)}}}%
}

\AddToHook{cmd/appendix/before}{\crefalias{section}{appendix}}

\title{CopT: Contrastive On-Policy Thinking with Continuous Spaces for General and Agentic Reasoning}

\author{Dachuan Shi$^1$, Hanlin Zhu$^2$, Xiangchi Yuan$^1$, Wanjia Zhao$^3$, Kejing Xia$^1$, \\ \textbf{Wen Xiao}$^4$, \textbf{Wenke Lee}$^1$\\
$^1$Georgia Tech \ \ $^2$UC Berkeley \ \ $^3$Stanford University \ \ $^4$Microsoft  \\
}

\begin{document}

\maketitle

\vspace{-1em}

\begin{figure}[h]
    \captionsetup{font={small}}
    \centering
    \includegraphics[width=1.0\linewidth]{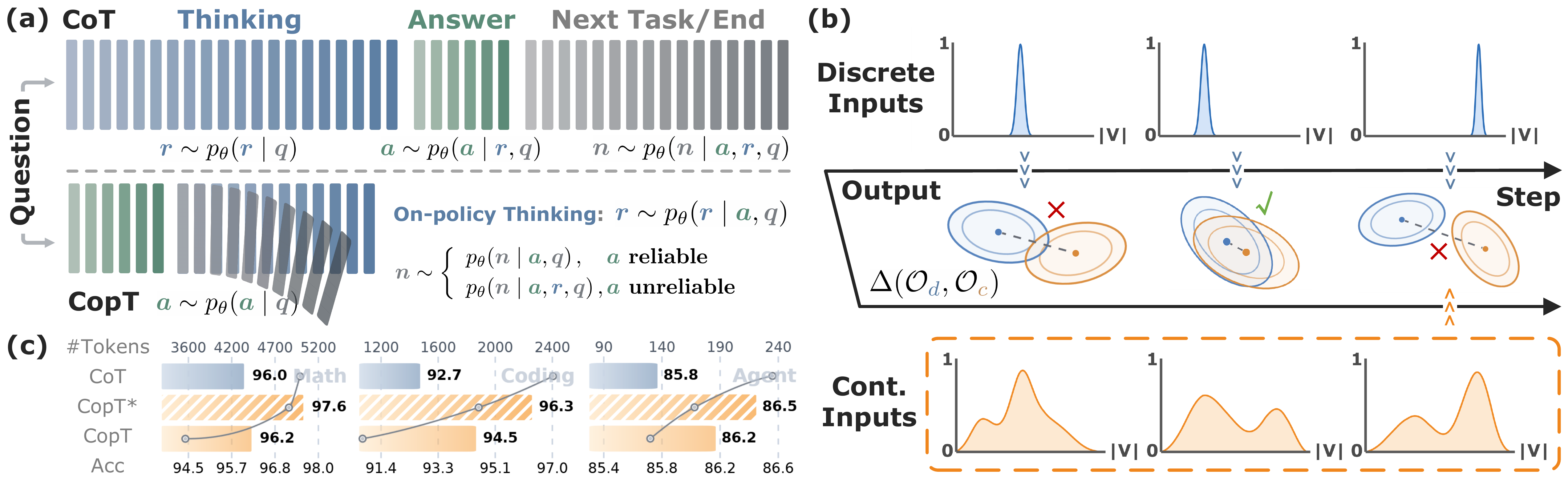}
    \caption{(a) Conceptual comparison between \textsc{CoT} thinking and \textsc{CopT} on-policy thinking. (b) \copt contrasts the output distributions under discrete and continuous inputs. (c) \copt improves peak accuracy, marked by $^{*}$, across mathematics, coding, and agentic reasoning tasks and nearly halves token usage at matched accuracy.}
    \label{fig:acc}
\end{figure}

\begin{abstract}
  Chain-of-thought (\textsc{CoT}) is a standard approach for eliciting reasoning capabilities from large language models (LLMs). However, the common \textsc{CoT} paradigm treats thinking as a prerequisite for answering, which can delay access to plausible answers and incur unnecessary token costs even when the model is able to identify an answer before extended thinking, a behavior known as performative reasoning. In this paper, we introduce \copt, a reformulated reasoning pipeline that reverses the usual order of thinking and answering. Instead of thinking before answering, \copt first elicits a draft answer and then invokes subsequent on-policy thinking conditioned on its own draft answer for reflection and correction. To assess whether the draft answer should be trusted, \copt recasts continuous embeddings as inference-time contrastive verifiers. Specifically, it contrasts the model's support for the same generated tokens under discrete-token inputs and continuous-embedding inputs, yielding a sequence-level reverse KL estimator for answer reliability. Our analysis shows that under certain assumptions, the expected estimate equals the mutual information between the unresolved latent state and the emitted answer token, explaining why it captures answer-relevant uncertainty rather than arbitrary uncertainty in the latent state. When the answer is deemed insufficiently reliable, \copt performs further on-policy thinking, where a second KL estimator dynamically controls draft-answer visibility, preserving useful partial information while reducing the risk of being misled by unreliable content. Across mathematics, coding, and agentic reasoning tasks, \copt improves peak accuracy by up to $23\%$ and reduces token usage by up to $57\%$ at comparable or higher accuracy, without any additional training. The code is available at \href{https://github.com/sdc17/CopT}{https://github.com/sdc17/CopT}.
\end{abstract}

\section{Introduction}

Reasoning has become one of the central capabilities of large language models (LLMs), enabling them to solve increasingly complex tasks in mathematics~\citep{deepmind2023imo,jaech2024openai,openai2025o3,team2025kimi}, coding~\citep{qwen_qwen3_coder_next_tech_report, hui2024qwen2, zhu2024deepseek, roziere2023code}, and agentic~\citep{anthropic2025opus45,anthropic2025claude4_system_card,qwen3.5,patil2024gorilla} settings. A common approach for eliciting reasoning behavior is chain-of-thought (\textsc{CoT}), where LLMs generate intermediate natural-language steps before producing the final answer~\citep{wei2022chain, yao2023tree, goyalthink, pfau2024let, qwen2.5, qwq32b}. By making the thinking process explicit, \textsc{CoT} brings substantial improvements to complex tasks that demand advanced reasoning capabilities~\citep{yang2025qwen3, meta2025llama, meta_llama_llama3_3_model_card, guo2025deepseek, agarwal2025gpt, abdin2025phi, pmlr-v267-shi25b, abouelenin2025phi}.

A key limitation of the predominant \textsc{CoT} paradigm is that it treats thinking as a prerequisite for answering. It works by first producing a thorough reasoning trace and only then arriving at the answer. However, recent work has revealed that, for many queries, LLMs exhibit performative reasoning~\citep{boppana2026reasoning, huang2026does, lindsey2026emergent, chen2025rmda}, in which they insist on completing the reasoning process even when they have already internally identified a plausible answer.

We propose \copt, a reversed reasoning paradigm. Rather than thinking before answering, an LLM first drafts an answer and then performs thinking for reflection and correction afterward. This reformulated paradigm provides earlier access to answers and avoids unnecessary token consumption when the model is able to identify a plausible answer before thorough thinking. 

Reversing the usual order of thinking and answering raises two key challenges: when a draft answer should be trusted, and how it should be used during later thinking. We show that continuous embeddings, previously used for generation in latent \textsc{CoT} methods~\citep{hao2024training, xu2025softcot}, can be recast as inference-time verifiers for this reversed reasoning setting. By contrasting the model's support for the same generated tokens under discrete-token and continuous-embedding inputs, they provide measurable criteria for draft reliability estimation and controlled utilization.

Latent \textsc{CoT}, where LLMs generate continuous embeddings instead of committing to discrete tokens during the thinking process~\citep{hao2024training, shen2025codi, zhu2025reasoning, xu2025softcot, tan2025think}, is a distinct line of recent work in parallel to explicit \textsc{CoT}. These approaches are motivated by the observation that latent \textsc{CoT} offers higher representational bandwidth
per step~\citep{zhu2025survey,yu2026latent}. Continuous embeddings can encode richer information by preserving uncertainty, whereas discrete tokens retain only the information carried by the sampled token at each step~\citep{li2025implicit, chen2025reasoning}.

Instead of using continuous embeddings for generation during the thinking process, as in existing latent reasoning methods, \copt keeps thinking explicit while recasting continuous embeddings as contrastive verifiers at inference time. This allows \copt to retain the readability of explicit \textsc{CoT} while simultaneously leveraging uncertainty information, as in latent \textsc{CoT}. Meanwhile, it avoids issues that may arise when continuous embeddings are directly used for generation, such as unseen representations~\citep{zhang2025softthinking}, de-diversification~\citep{liang2025singlethread}, and drifting into noise~\citep{shi2025swireasoning}.

To address the first challenge of determining when a draft answer should be trusted, \copt introduces a contrastive mechanism with continuous spaces to estimate the reliability of the draft answer. Specifically, it contrasts the model's support for its own generated answer under two types of input representations: explicit inputs in discrete spaces and continuous embeddings constructed from next-token distributions and cached online along with explicit token generation. This contrast yields a sequence-level reverse KL estimator that indicates the reliability of the draft answer. If the draft answer appears sufficiently reliable, it will be accepted by the model directly. Otherwise, \copt triggers a subsequent on-policy thinking process to either correct or support the answer.

The second challenge of how to use the draft answer arises once on-policy thinking is triggered. A draft answer deemed insufficiently reliable may still contain useful partial information, but exposing it throughout the entire later thinking process risks misleading the model. To control the visibility of the draft answer during on-policy thinking across thinking steps, \copt periodically calculates a second KL estimator within each thinking chunk using a similar contrastive mechanism with continuous spaces. In this way, \copt allows the model to use the draft answer when it appears helpful, while hiding it when the current thinking process becomes unstable.

Beyond the empirical results, we further provide a latent-state interpretation of the proposed contrastive estimator. Under a local mixture-prefix view, the continuous prefix preserves uncertainty over an unresolved latent reasoning state $S$, while the emitted answer token is denoted by $A$. We show that under a mixture-linear assumption (see \Cref{sec:theory}), the expected estimate equals the mutual information $I(S;A)$, indicating that the estimator measures answer-relevant uncertainty rather than the entropy of the latent state itself. This explains why the score grows only when uncertainty preserved by continuous embeddings changes the model's support for its own generated tokens, supporting its use for draft reliability estimation.

Our contributions are summarized as follows:
\begin{itemize}[leftmargin=*,labelsep=0.5em]
  \item We propose \copt, a training-free reasoning pipeline that enables LLMs to start with a draft answer and invoke on-policy thinking conditioned on it when necessary, thereby allowing earlier access to answers and selective correction afterward.
  \item We introduce a contrastive mechanism that measures the discrepancy between the model’s support for the same generated tokens under discrete and continuous inputs, which helps identify potential errors in draft answers and modulates their exposure during the thinking process.
  \item We extensively validate the effectiveness of \copt on mathematics, coding, and agentic reasoning tasks across multiple benchmarks, model architectures, and scales, demonstrating consistent gains over \textsc{CoT} baselines in both accuracy and token efficiency.
\end{itemize}

\section{Related Work}

\paragraph{Reasoning LLMs and explicit reasoning.}
Reasoning with explicit natural-language traces has become a standard way to improve the performance of LLMs on complex tasks~\citep{openai2025o3, anthropic2025opus45, comanici2025gemini}. Early work elicits such behavior through prompting~\citep{wei2022chain,wang2022self,yao2023tree}. More recent LLMs typically gain reasoning capabilities through reinforcement learning~\citep{shao2024deepseekmath,yu2025dapo,liu2025understanding} or multi-stage post-training that combines supervised fine-tuning with reinforcement learning~\citep{liu2024deepseek,pmlr-v235-shi24e,yang2025qwen3,ma2025learning,yuan2025mitigating,yuan2025superficial}. Representative open-source examples include DeepSeek-R1~\citep{guo2025deepseek} and Qwen3~\citep{yang2025qwen3}, which show that large-scale reinforcement learning and long-\textsc{CoT} post-training can elicit strong reasoning behaviors. Following works~\citep{zeng2025glm, liu2025deepseek, yuan2026behavior, qwen_qwen3_coder_next_tech_report} further demonstrate the effectiveness of explicit reasoning across diverse mathematics, coding, and agentic tasks. Despite these advances, reasoning LLMs typically retain the standard thinking-before-answering order. In contrast, \copt reverses this order by first eliciting a draft answer and invoking on-policy thinking conditioned on it when the answer appears insufficiently reliable.

\paragraph{Latent reasoning with continuous embeddings.}
A parallel line of work explores latent reasoning in continuous spaces, where LLMs operate on continuous embeddings instead of committing to discrete tokens at every reasoning step~\citep{hao2024training, su2025token, zhu2025reasoning}. These methods are motivated by the observation that continuous representations can encode information from the full next-token distribution, while discrete decoding retains only the sampled token. Latent reasoning is mainly achieved by adapting LLMs into continuous spaces via modified pretraining~\citep{zeng2025pretraining,tack2025llm} or fine-tuning ~\citep{shen2025codi, xu2025softcot, tan2025think,wei2025sim,zhu2025emergence,xia2026metastate} objectives. Recent training-free methods~\citep{liang2025singlethread,xu2026thinking} instead construct continuous embeddings directly during inference, such as Soft-Thinking~\citep{zhang2025softthinking} and SwiReasoning~\citep{shi2025swireasoning}. These prior latent reasoning methods mainly use continuous embeddings as a medium for generation. In contrast, \copt recasts them as inference-time verifiers. This allows \copt to use uncertainty information preserved by continuous embeddings as in latent reasoning while retaining the readability of explicit reasoning.

\section{Methodology}

\label{sec:method}

As shown in Fig.~\ref{fig:method}, \copt reformulates LLM reasoning into two reversed stages: a leading \emph{draft-answer stage} and, when necessary, a trailing \emph{on-policy thinking stage}. The key insight is to first elicit an early-stage answer at low cost, estimate its reliability with a normalized sequence-level reverse KL estimator, and selectively trigger on-policy thinking with dynamic access to the draft answer.

\begin{figure}[t]
    \captionsetup{font={small}}
    \centering
    \includegraphics[width=1.0\linewidth]{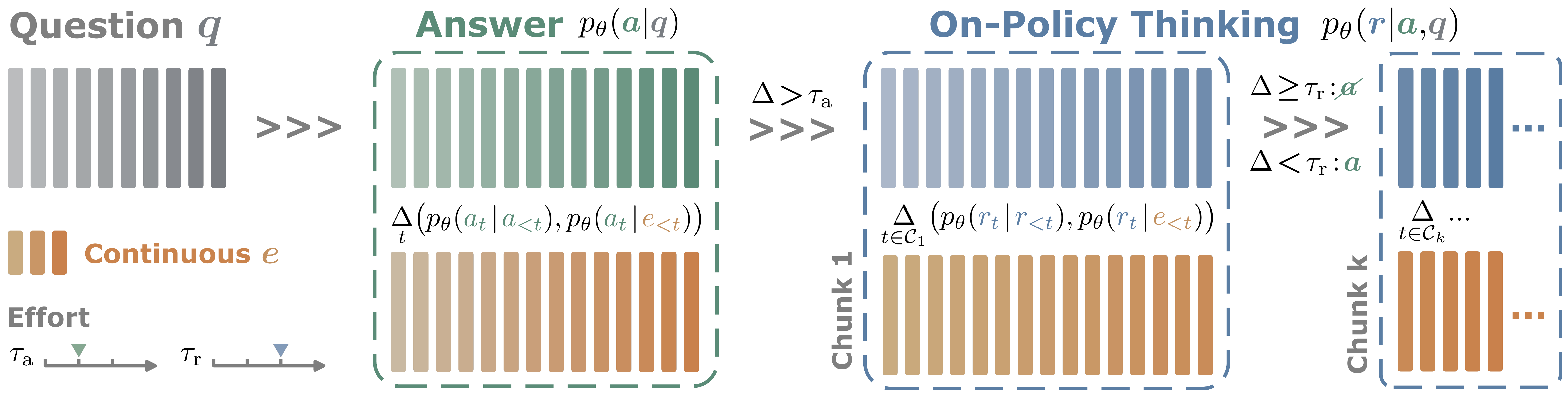}
    \vspace{-1em}
    \caption{\copt starts with a draft answer and performs on-policy thinking conditioned on it. It contrasts the model's support for the same chosen tokens under discrete and continuous inputs to estimate draft answer reliability, and during thinking, chunk by chunk, to determine the visibility of the draft answer across time steps.}
    \label{fig:method}
\end{figure}

\subsection{Reliability Estimator of the Draft Answer}

\paragraph{Draft answer elicitation. } 

Let $p_\theta$ denote the model with parameters $\theta$. Let $E \in \mathbb{R}^{|\mathcal{V}| \times d}$ be the input embedding matrix of the model, where $\mathcal{V}$ is the vocabulary and $d$ is the hidden size. For any token $v \in \mathcal{V}$, $E(v)\in\mathbb{R}^d$ denotes its embedding. Given a question token sequence $q = (q_1,\dots,q_m)$, instead of allowing the model to think thoroughly, we force it to output \texttt{</think>} at the beginning and go straight into its answering mode.

\paragraph{Reliability estimation. } To estimate how likely it is that a subsequent thorough thinking process will be required for correcting potential errors, we introduce a normalized sequence-level reverse KL estimator $\kappa_{\text{a}}$. Let the draft phase generate tokens
$a = (a_1,\dots,a_{T_a})$. During draft answer generation, for each generated token $a_t$, we cache two items calculated from the next-token distributions:
\[
p_t := p_\theta(a_t \mid q, a_{<t}),
\qquad
e_t := \sum_{v \in \mathcal{V}} p_\theta(v \mid q, a_{<t}) E(v).
\]
Here $p_t$ is the chosen-token probability, and $e_t$ is a continuous embedding obtained as the probability-weighted average over the vocabulary, which preserves uncertainty information at each step.

After the draft answer $a$ is completed, we calculate $\kappa_{\text{a}}$ to estimate its reliability. More specifically, we compare the student discrete-prefix distribution induced by the original inputs against the teacher continuous-prefix distribution in which inputs are replaced with cached continuous embeddings. 

For $t \in \{1, 2, \ldots T_a\}$, the student probability is simply $p_t$, and all the teacher probabilities are obtained in parallel using a single forward pass with the modified input embeddings. The teacher probabilities of the original answer tokens are gathered at the corresponding output positions:
\[
p_t^{e}
:=
p_\theta(a_t \mid q, e_{<t}).
\]
This defines the continuous-prefix probability of the sampled draft answer as $p_\theta^{e}(a|q) = \prod_{t=1}^{T_a} p_t^e$.

We define the estimator
\[
\kappa_{\text{a}}(a_{1:T_a})
:=
\frac{1}{T_a}
\sum_{t=1}^{T_a}
\left[
\log p_\theta(a_t \mid q, a_{<t})
-
\log p_\theta(a_t \mid q, e_{<t})
\right].
\]
For any fixed draft length $T_a$, $\kappa_{\text{a}}$ is an unbiased estimator of the normalized sequence-level reverse KL divergence between the two distributions:
\[
\mathbb{E}_{a_{1:T_a}\sim p_\theta(\cdot \mid q)}
\bigl[\kappa_{\text{a}}(a_{1:T_a})\bigr]
=
\frac{1}{T_a}
D_{\mathrm{KL}}
\!\left(
p_\theta(\cdot \mid q)
\,\middle\|\,
p_\theta^{e}(\cdot \mid q)
\right).
\]

\subsection{On-Policy Thinking Conditioned on the Draft Answer}

\paragraph{On-policy thinking elicitation. } 

A large $\kappa_{\text{a}}$ indicates that answer context becomes substantially less supported with teacher-forced continuous embeddings, \textit{i.e.}, the answer may be unreliable given additional uncertainty information. Let $\tau_{\text{a}}$ denote the reliability threshold. When $\kappa_{\text{a}} > \tau_{\text{a}}$, we force the model to output \texttt{<think>} after the draft answer and move into a subsequent thinking process.

\paragraph{Visibility controls for draft answers.}

For draft answers that are deemed insufficiently reliable, the goal of on-policy thinking is to use any beneficial information when necessary, while avoiding being misled by unreliable draft content. Let the on-policy thinking phase generate tokens
\[
r = (r_{T_a+1}, r_{T_a+2}, \dots, r_{T_a+T_r}).
\]
We partition the thinking trajectory into chunks of length $C$. Let the $k$-th chunk start at position
\[
s_k := T_a + 1 + (k-1)C,
\]
and span positions $[s_k, s_k + C - 1]$. Let
\[
m_k := \mathbbm{1}\{a \text{ is visible in chunk } k\}
\]
denote the visibility mask for the $k$-th chunk, and define the visibility-conditioned draft input as
\[
a^{(m_k)}
:=
\begin{cases}
a, & m_k = 1, \\
\varnothing, & m_k = 0.
\end{cases}
\]
For each generated token $r_t$ in chunk $k$, we cache
\[
p_t
:=
p_\theta(r_t \mid q, a^{(m_k)}, r_{T_a+1:t-1}),
\qquad
e_t
:=
\sum_{v \in \mathcal{V}} p_\theta(v \mid q, a^{(m_k)}, r_{T_a+1:t-1}) E(v).
\]
Similarly, we calculate a second KL estimator $\kappa_{\text{r}}^{(k)}$ on the current chunk whenever it reaches a predefined length $C$ to decide whether the previous draft answer should become visible in the next chunk. \footnote{$\kappa_{\text{a}}$ and $\kappa_{\text{r}}$ are calculated on the already generated sequence, and therefore incur only small overhead once the corresponding chosen-token probabilities and continuous embeddings are cached online during generation.} For $t \in \{s_k, \ldots,  s_k+C-1\}$, the student chosen-token probability is simply $p_t$, and all the teacher probabilities within the chunk are obtained in parallel using a single forward pass with the modified intra-chunk input embeddings:
\[
p_t^{e}
:=
p_\theta(r_t \mid q, a^{(m_k)}, r_{T_a+1:s_k-1}; e_{s_k:t-1}).
\]
We define the estimator
\begin{multline*}
\kappa_{\text{r}}^{(k)}(r_{s_k:s_k+C-1})
:=
\frac{1}{C}
\sum_{t=s_k}^{s_k+C-1}
\bigl[
\log p_\theta(r_t \mid q, a^{(m_k)}, r_{T_a+1:t-1})
- \\
\log p_\theta(r_t \mid q, a^{(m_k)}, r_{T_a+1:s_k-1}; e_{s_k:t-1})
\bigr].
\end{multline*}
For a fixed chunk length $C$, $\kappa_{\text{r}}^{(k)}$ is an unbiased estimator of the normalized sequence-level reverse KL between the two chunk-level continuation distributions:
\begin{multline*}
\mathbb{E}_{r_{s_k:s_k+C-1}\sim p_\theta(\cdot \mid q, a^{(m_k)}, r_{T_a+1:s_k-1})}
\bigl[\kappa_{\text{r}}^{(k)}(r_{s_k:s_k+C-1})\bigr] = \\
\frac{1}{C}\,
D_{\mathrm{KL}}\!\Bigl(
p_\theta(\cdot \mid q, a^{(m_k)}, r_{T_a+1:s_k-1})
\big\|\,
p_\theta^{e}(\cdot \mid q, a^{(m_k)}, r_{T_a+1:s_k-1})
\Bigr).
\end{multline*}

$\kappa_{\text{r}}^{(k)}$ estimates the reliability of the current thinking chunk. A large $\kappa_{\text{r}}^{(k)}$ suggests that the current chunk is unstable and more vulnerable to misleading information in the draft answer. Let $\tau_{\text{r}}$ denote the stability threshold. After each complete chunk, we update the visibility of the draft answer for the next chunk by
\[
m_{k+1} =
\begin{cases}
1, & \kappa_{\text{r}}^{(k)} < \tau_{\text{r}}, \\
0, & \text{otherwise}.
\end{cases}
\]

\section{Theoretical Analysis}
\label{sec:theory}

In this section, we provide a theoretical interpretation to demonstrate the effectiveness of our \copt method under certain assumptions. We focus on the reliability of our proposed reverse-KL estimator. 
Our analysis highlights a key property of the reverse-KL estimator: it measures \emph{answer-relevant uncertainty}, rather than uncertainty over latent reasoning states themselves.


For convenience, we analyze a single answer position. Note that all probability distributions below are conditioned on the question (or equivalently, the prompt) $q$ and the previous output prefix, which we omit when the context is clear. Let $\mathcal{S}$ be a finite set of latent reasoning states. A discrete output prefix (along with the prompt) commits the model to one latent state, while a continuous prefix may represent a superposition of several possible states.

Let $\mathcal{A}$ be a finite set of all possible answers (or equivalently, the next token). When the prefix is discrete, for each latent state $s\in\mathcal{S}$, let
\[
P_s(a)
=
p_\theta(a\mid s),  \quad \forall a \in \mathcal{A}
\]
denote the next-token distribution induced by committing to $s$, where $\theta$ denotes the model parameters. When the prefix is continuous, we make the following assumption on the output distribution.

\begin{assumption}[Mixture-linear continuous prefix]
\label{assump:mixture_soft_prefix}
Let $w$ be a distribution over $\mathcal{S}$ such that the discrete draft prefix commits to a latent state
$S\sim w$,
and then emits the answer
$A\sim P_S$.
Let $e_w$ denote the corresponding continuous prefix which is determined by the distribution $w$. We assume the next-token distribution conditioned on a continuous prefix $e_w$ is determined by 
\[
p_\theta(a\mid e_w)
=
\bar P_w(a)
:=
\sum_{s\in\mathcal{S}}w(s)P_s(a), \quad \forall a \in \mathcal{A}.
\]
\end{assumption}

Note that for the emitted answer token $A$, the local reverse-KL contribution is
\[
\kappa(S,A)
=
\log P_S(A)-\log \bar P_w(A).
\]

\begin{theorem}[Reverse KL measures answer-relevant uncertainty]
\label{thm:answer_relevant_uncertainty_main}
Under Assumption~\ref{assump:mixture_soft_prefix},
\[
\mathbb{E}_{S\sim w,\,A\sim P_S}
[\kappa(S,A)]
=
\sum_{s\in\mathcal{S}}w(s)
D_{\mathrm{KL}}
\!\left(
P_s
\,\middle\|\,
\bar P_w
\right) = 
I(S;A),
\]
where $I(S;A)$ is the mutual information between the latent state $S$ and the emitted answer token $A$ under the joint distribution
\[
\Pr(S=s,A=a)=w(s)P_s(a).
\]
\end{theorem}

The proof is deferred to \Cref{app:answer_relevant_uncertainty}.
\Cref{thm:answer_relevant_uncertainty_main} shows that \copt does not penalize latent-state uncertainty by itself. Instead, it measures whether that uncertainty changes the next answer-token distribution. For example, the continuous prefix may represent a mixture over several possible states,
$S\in\{s_1,s_2,s_3\}$,
which can have high entropy. However, if all three states induce the same next answer token or the same next-token distribution, then the emitted token $A$ carries no information about which state was selected. In that case,
$I(S;A)=0$,
and the expected reverse-KL contribution is zero. Thus, high uncertainty over latent states is harmless when all plausible states agree on the next answer.

Applying this argument token by token, if the mixture-prefix assumption holds at each answer position $t$, then the normalized draft score satisfies
\[
\mathbb{E}[\kappa_{\mathrm a}]
=
\frac{1}{T_a}
\sum_{t=1}^{T_a}
I(S_t;A_t),
\]

which is conditioned on the preceding context at each position. Therefore, $\kappa_{\mathrm a}$ estimates the average amount of answer-relevant uncertainty in the draft answer.

\section{Experiments}
\label{sec:experiments}

\subsection{Experimental Settings}

\paragraph{Models.} 
We evaluate \copt on pure Transformer-based Qwen3 models~\citep{yang2025qwen3} and hybrid Gated-DeltaNet Qwen3.5 models~\citep{qwen3.5} at 2B, 8B, and 35B scales. This selection allows us to validate the effectiveness of \copt across model families, scales, and architectures, including pure Transformer, hybrid, dense, and sparse mixture-of-experts models.

\paragraph{Domains and Benchmarks.} We evaluate \copt on 10 benchmarks spanning four domains: math and STEM reasoning~(GSM8K~\citep{cobbe2021training}, Math500~\citep{hendrycks2021measuring}, AIME 2024~\citep{hf_aime24_dataset}, AIME 2025~\citep{hf_aime25_dataset}, GPQA Diamond~\citep{rein2024gpqa}); coding reasoning~(HumanEval~\citep{chen2021evaluating}, MBPP~\citep{austin2021program}, LeetCode-Contest~\citep{guo2024deepseekcoder}); single-turn and multi-turn agentic reasoning~(BFCL v4~\citep{patil2025bfcl}, ZebraArena~\citep{zhao2026zebraarena}). More details are provided in Appendix~\ref{supp_benchmark}.

\begin{table}[t]
\centering
\caption{Comparison on mathematics, coding, and STEM reasoning benchmarks with the Qwen3-8B model. \textcolor{oliveGreen}{Green blocks} indicate increasing reasoning effort of \copt to achieve higher accuracy.}
\vspace{0.8em}
\begingroup
\footnotesize
\setlength{\tabcolsep}{1.6pt}
\begin{tabular}{lcccccccc}
\toprule[1.2pt]
\multirow[c]{2}{*}{\raisebox{-0.5ex}{\textbf{Mathematics}}}
& \multicolumn{2}{c}{\textbf{GSM8K}}
& \multicolumn{2}{c}{\textbf{Math500}}
& \multicolumn{2}{c}{\textbf{AIME24}}
& \multicolumn{2}{c}{\textbf{AIME25}} \\
\cmidrule(lr){2-3}
\cmidrule(lr){4-5}
\cmidrule(lr){6-7}
\cmidrule(lr){8-9}
& Acc.~(\%) & \# Tokens
& Acc.~(\%) & \# Tokens
& Acc.~(\%) & \# Tokens
& Acc.~(\%) & \# Tokens \\
\midrule[0.8pt]
CoT
& 95.75 & 2138
& 96.00 & 4985
& 75.83 & 12077
& 67.50 & 12924 \\
CoT~(Greedy)
& 95.83 & 2240
& 96.40 & 5311
& 70.00 & 11680
& 60.00 & 13292 \\
\midrule[0.5pt]
\multirow{2}{*}{\textbf{CopT~(Ours)}}
& \cellcolor{coptMainBg}\accup{96.36}{0.61}
& \cellcolor{coptMainBg}\tokdown{1813}{-15.2}
& \cellcolor{coptMainBg}\accup{97.60}{1.60}
& \cellcolor{coptMainBg}\tokdown{4851}{-2.7}
& \multirow{2}{*}{\accup{79.17}{3.34}}
& \multirow{2}{*}{\tokdown{11525}{-4.6}}
& \multirow{2}{*}{\accup{70.42}{2.92}}
& \multirow{2}{*}{\tokdown{12801}{-1.0}} \\
& \accup{95.98}{0.23}
& \tokdown{961}{-55.1}
& \accup{96.20}{0.20}
& \tokdown{3609}{-27.6}
&  & 
&  &  \\
\midrule[1.0pt]
\multirow[c]{2}{*}[-0.8ex]{\shortstack[l]{\textbf{Coding}\\[-1pt]\textbf{\& STEM}}}
& \multicolumn{2}{c}{\textbf{HumanEval}}
& \multicolumn{2}{c}{\textbf{LeetCode-Contest}}
& \multicolumn{2}{c}{\textbf{MBPP}}
& \multicolumn{2}{c}{\textbf{GPQA Diamond}} \\
\cmidrule(lr){2-3}
\cmidrule(lr){4-5}
\cmidrule(lr){6-7}
\cmidrule(lr){8-9}
& Acc.~(\%) & \# Tokens
& Acc.~(\%) & \# Tokens
& Acc.~(\%) & \# Tokens
& Acc.~(\%) & \# Tokens \\
\midrule[0.8pt]
CoT
& 92.68 & 2368
& 59.44 & 7306
& 94.16 & 2033
& 59.60 & 8123 \\
CoT~(Greedy)
& 93.90 & 2627
& 57.22 & 6975
& 91.44 & 2724
& 56.57 & 7909 \\
\midrule[0.5pt]
\multirow{2}{*}{\textbf{CopT~(Ours)}}
& \cellcolor{coptMainBg}\accup{96.34}{3.66}
& \cellcolor{coptMainBg}\tokdown{1842}{-22.2}
& \cellcolor{coptMainBg}\accup{66.11}{6.67}
& \cellcolor{coptMainBg}\tokup{7607}{4.1}
& \multirow{2}{*}{\accup{94.55}{1.39}}
& \multirow{2}{*}{\tokdown{1997}{-1.8}}
& \multirow{2}{*}{\accup{61.62}{2.02}}
& \multirow{2}{*}{\tokdown{6851}{-15.7}} \\
& \accup{94.51}{1.83}
& \tokdown{1023}{-56.8}
& \accup{61.11}{1.67}
& \tokdown{6993}{-4.3}
&  & 
&  &  \\
\bottomrule[1.2pt]
\end{tabular}
\endgroup
\label{tab:main_math_coding}
\end{table}

\subsection{Experimental Results on Mathematics and Coding Reasoning}

Tab.~\ref{tab:main_math_coding} reports accuracy and generation length on mathematics, coding, and STEM reasoning benchmarks with the Qwen3-8B model. Compared with standard \textsc{CoT} and greedy \textsc{CoT}, \copt improves accuracy while effectively reducing generation length across most settings. When applicable, we report two sets of \copt results: one targeting accuracy comparable to or higher than \textsc{CoT}, and another, shown in green, that further improves peak accuracy by increasing reasoning effort.

On mathematics benchmarks, the token-saving setting of \copt improves GSM8K accuracy by $+0.23\%$ while reducing generated tokens by $55.1\%$, and improves
Math500 accuracy by $+0.20\%$ while reducing generated tokens by $27.6\%$. These results show substantial efficiency gains on problems that do not require extended thinking. With increasing reasoning effort, \copt further improves GSM8K and Math500 accuracy by $+0.61\%$ and $+1.60\%$, respectively. On more challenging AIME benchmarks, \copt obtains larger accuracy gains: $+3.34\%$ on AIME24 and $+2.92\%$ on AIME25. The same trend holds on coding and STEM tasks. At matched accuracy levels, \copt improves HumanEval accuracy by $+1.83\%$ while reducing tokens by $56.8\%$. With increasing reasoning effort, \copt achieves larger accuracy gain of $+3.66\%$, $+6.67\%$, $+1.39\%$, $+2.02\%$ on HumanEval, LeetCode-Contest, MBPP, and GPQA Diamond, respectively.

These results suggest that \copt improves peak accuracy by selectively invoking on-policy thinking when the draft answer appears insufficiently reliable. This is especially beneficial on harder benchmarks such as AIME24, AIME25, LeetCode-Contest, and GPQA Diamond, where draft answers are more likely to require correction. In addition, \copt reduces unnecessary thinking on easier examples by allowing sufficiently reliable draft answers to be accepted earlier. This leads to considerable token savings on benchmarks such as GSM8K and HumanEval. 

\begin{table}[t]
\centering
\caption{Comparison with training-free methods that use continuous embeddings for generation. Token counts are measured by generation steps, regardless of whether the steps are explicit or latent.}
\vspace{0.3em}
\begingroup
\footnotesize
\setlength{\tabcolsep}{3.5pt}
\begin{tabular}{lccccccccc}
\toprule[1.2pt]
\multirow[c]{2}{*}{\raisebox{-0.5ex}{\textbf{Method}}}
& \multicolumn{2}{c}{\textbf{GSM8K}}
& \multicolumn{2}{c}{\textbf{AIME25}}
& \multicolumn{2}{c}{\textbf{HumanEval}}
& \multicolumn{2}{c}{\textbf{GPQA Diamond}}
& \multirow[c]{2}{*}[-0.7ex]{\textbf{\shortstack{Fully Explicit\\Readability}}} \\
\cmidrule(lr){2-3}
\cmidrule(lr){4-5}
\cmidrule(lr){6-7}
\cmidrule(lr){8-9}
& Acc.~(\%) & \# Tokens
& Acc.~(\%) & \# Tokens
& Acc.~(\%) & \# Tokens
& Acc.~(\%) & \# Tokens
&  \\
\midrule[0.8pt]
Soft-Thinking
& 95.38 & 2073
& 68.33 & 13665
& 92.07 & 2408
& 59.60 & 8153
& \textcolor{brickRed}{$\times$} \\
SwiReasoning
& 96.06 & 2218
& 70.00 & 13911
& 95.73 & 2894
& 61.11 & 8359
& \textcolor{brickRed}{$\times$} \\
\midrule[0.5pt]
\textbf{CopT~(Ours)}
& \textbf{96.36}
& \textbf{1813}
& \textbf{70.42}
& \textbf{12801}
& \textbf{96.34}
& \textbf{1842}
& \textbf{61.62}
& \textbf{6851}
& \textcolor{oliveGreen}{$\boldsymbol{\checkmark}$} \\
\bottomrule[1.2pt]
\end{tabular}
\endgroup
\label{tab:main_baselines}
\end{table}

\subsection{Comparison with Training-Free Continuous-Generation Methods}

Tab.~\ref{tab:main_baselines} compares \copt with two training-free methods that directly use continuous embeddings for generation, Soft-Thinking~\citep{zhang2025softthinking} and SwiReasoning~\citep{shi2025swireasoning}, with the Qwen3-8B model. \copt differs from them by keeping generation fully explicit and recasting continuous embeddings as inference-time verifiers.  Overall, \copt achieves the best accuracy with the fewest generated tokens. Compared with SwiReasoning, \copt improves accuracy by $+0.30\%$, $+0.42\%$, $+0.61\%$, and $+0.51\%$ on four mathematics, coding, and STEM benchmarks, while using $18.3\%$, $8.0\%$, $36.4\%$, and $18.0\%$ fewer tokens. In addition, the reasoning processes of these latent \textsc{CoT} methods are not fully represented in natural language. In contrast, \copt allows users to inspect the complete reasoning process while still benefiting from uncertainty information preserved in continuous embeddings. These results indicate that continuous embeddings do not need to be used as a generation strategy to improve reasoning. By recasting them as contrastive verifiers, \copt better balances accuracy, efficiency, and readability.

\subsection{Controllable Reasoning Effort and Latency Reduction}

\begin{figure}[t]
    \vspace{-0.5em}
    \centering
    \includegraphics[width=0.9\linewidth]{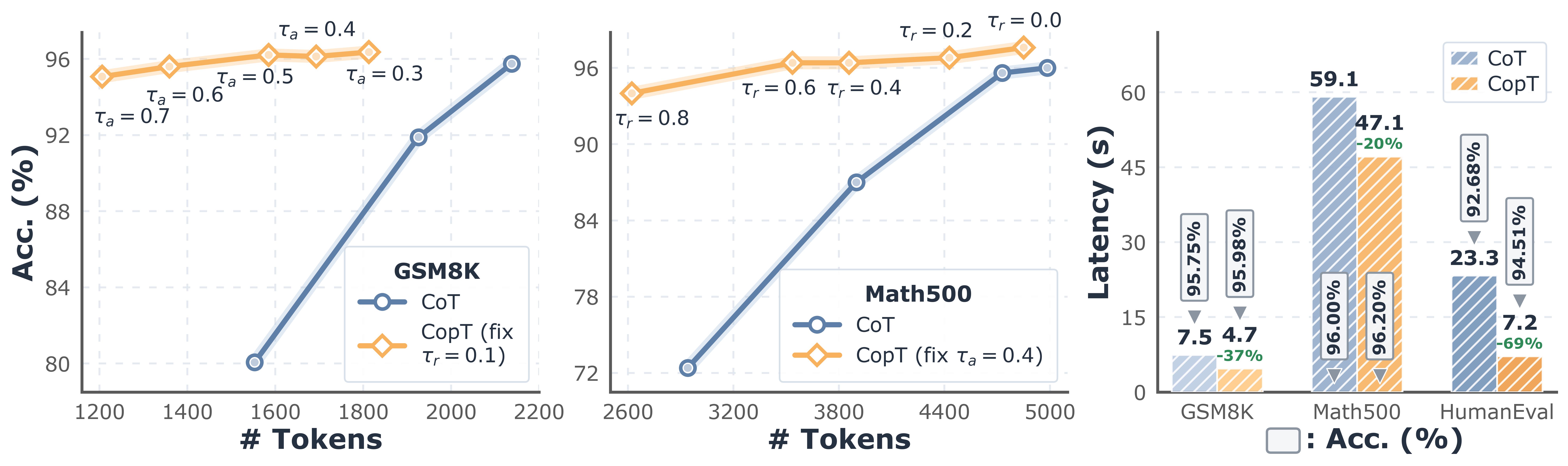}
    \vspace{-0.5em}
    \caption{Left and center: Controllable reasoning effort by $\tau_a$ and $\tau_r$. Right: \copt reduces average per-sample latency at comparable or higher accuracy, measured on a single H200 GPU.}
    \label{fig:efficiency}
    \vspace{-0.5em}
\end{figure}

Fig.~\ref{fig:efficiency} studies how \copt controls reasoning effort and how token savings translate into real latency reduction. The left and center panels show the accuracy-token trade-off obtained by sweeping $\tau_{\text a}$ and $\tau_{\text r}$. These curves show that \copt can reduce generation length when less thinking is needed, while allocating more on-policy thinking to obtain higher accuracy when the thresholds trigger stronger verification and correction. The right panel reports average latency per sample at comparable or higher accuracy. The reductions are consistent with the token-efficiency gains: \copt reduces latency by $37\%$ on GSM8K, $20\%$ on Math500, and $69\%$ on HumanEval. These results show that \copt reduces actual generation latency while preserving reasoning accuracy.

\begin{table}[t]
\centering
\caption{Comparison on agentic reasoning benchmarks BFCL v4~(single-turn) and ZebraArena~(multi-turn). \textcolor{oliveGreen}{Green blocks} indicate increasing reasoning effort of CopT to achieve higher accuracy.}
\vspace{0.3em}
\begingroup
\footnotesize
\setlength{\tabcolsep}{1.3pt}
\begin{tabular}{lcccccccccc}
\toprule[1.2pt]
\multirow[c]{4}{*}[-1.6ex]{\textbf{Method}}
& \multicolumn{2}{c}{\textbf{Qwen3.5-2B}}
& \multicolumn{8}{c}{\textbf{Qwen3.5-35B-A3B}} \\
\cmidrule(lr){2-3}
\cmidrule(lr){4-11}

& \multicolumn{2}{c}{\textbf{BFCL v4}}
& \multicolumn{2}{c}{\textbf{BFCL v4}}
& \multicolumn{6}{c}{\textbf{ZebraArena (multi-turn)}} \\
\cmidrule(lr){6-11}

& \multicolumn{2}{c}{\textbf{(non-live and live)}}
& \multicolumn{2}{c}{\textbf{(non-live and live)}}
& \multicolumn{2}{c}{\textbf{Small}}
& \multicolumn{2}{c}{\textbf{Medium}}
& \multicolumn{2}{c}{\textbf{Large}} \\
\cmidrule(lr){2-3}
\cmidrule(lr){4-5}
\cmidrule(lr){6-7}
\cmidrule(lr){8-9}
\cmidrule(lr){10-11}

& Acc.~(\%) & \# Tokens
& Acc.~(\%) & \# Tokens
& Acc.~(\%) & \# Tokens
& Acc.~(\%) & \# Tokens
& Acc.~(\%) & \# Tokens \\
\midrule[0.8pt]
CoT
& 77.53 & 234
& 85.77 & 235
& 93.71 & 3357
& 75.00 & 7217
& 59.21 & 8070 \\
\midrule[0.5pt]
\multirow{2}{*}{\textbf{CopT~(Ours)}}
& \cellcolor{coptMainBg}\accup{78.37}{0.84}
& \cellcolor{coptMainBg}\tokdown{164}{-29.9}
& \cellcolor{coptMainBg}\accup{86.45}{0.68}
& \cellcolor{coptMainBg}\tokdown{168}{-28.5}
& \multirow[c]{2}{*}[-0.4ex]{\accupbr{96.69}{2.98}}
& \multirow[c]{2}{*}[-0.4ex]{\tokupbr{3486}{3.8}}
& \multirow[c]{2}{*}[-0.4ex]{\accupbr{88.14}{13.14}}
& \multirow[c]{2}{*}[-0.4ex]{\tokdownbr{5457}{-24.4}}
& \multirow[c]{2}{*}[-0.4ex]{\accupbr{82.24}{23.03}}
& \multirow[c]{2}{*}[-0.4ex]{\tokdownbr{6486}{-19.6}} \\
& \accup{78.01}{0.48}
& \tokdown{139}{-40.6}
& \accup{86.17}{0.40}
& \tokdown{130}{-44.7}
& 
& 
& 
& 
& 
&  \\
\bottomrule[1.2pt]
\end{tabular}
\endgroup
\label{tab:main_agentic}
\end{table}

\subsection{Experimental Results on Agentic Reasoning}

Tab.~\ref{tab:main_agentic} evaluates \copt on agentic reasoning benchmarks, including the non-live and live splits of BFCL v4 and multi-turn ZebraArena with one missing clue. We use Qwen3.5 models for agentic evaluation because they are designed for stronger agentic capabilities than reasoning-only models~\citep{qwen3.5}. On BFCL v4, \copt consistently improves accuracy while reducing generation length across scales. In particular, \copt maintains comparable or higher accuracy with $40.6\%$ fewer tokens on Qwen3.5-2B and $44.7\%$ fewer tokens on Qwen3.5-35B-A3B. The gains are stronger on the multi-turn ZebraArena benchmark. \copt improves accuracy by $+2.98\%$, $+13.14\%$, and $+23.03\%$ on the small, medium, and large splits, respectively. For generation length, \copt uses $3.8\%$ more tokens on the small split, but reduces tokens by $24.4\%$ on the medium split and $19.6\%$ on the large split. These results suggest that \copt is especially useful in longer agentic interactions, where improved accuracy-token trade-offs can accumulate across multiple rounds of interaction. For simplicity, we use the same reasoning effort of $\tau_r=0, \tau_a=0.3$ for all single-turn BFCL v4 splits and $\tau_r=0, \tau_a=1.5$ for all multi-turn ZebraArena splits. Better trade-offs could be achieved by setting separate reasoning effort for different task difficulties.

\subsection{Ablation Studies on Design Choices}

\begin{figure}[t]
    \vspace{-0.1em}
    \centering
    \includegraphics[width=1.0\linewidth]{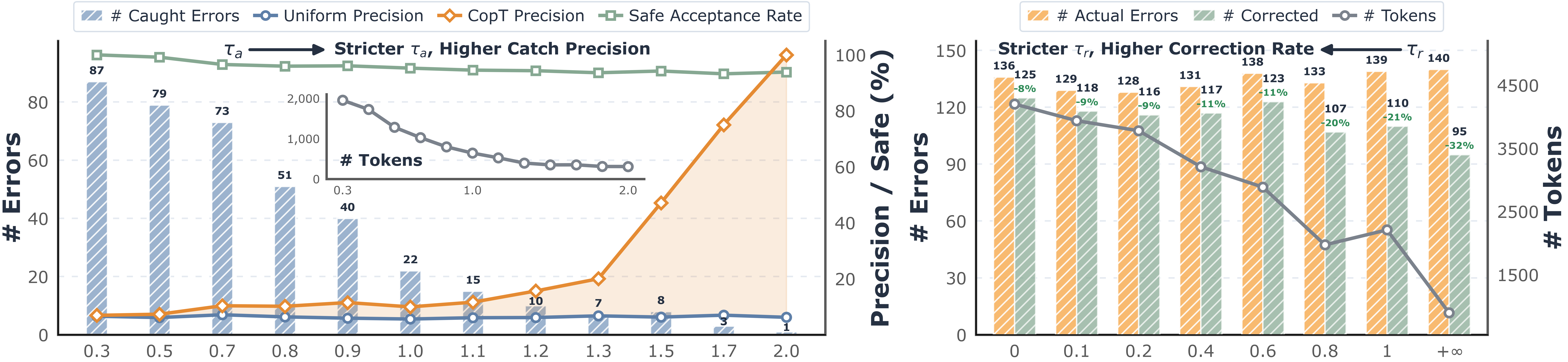}
    \vspace{-1em}
    \caption{Left: The estimator $\kappa_a$ identifies erroneous drafts more precisely than uniform selection. Right: The threshold $\tau_r$ trades off correction rate and token usage by controlling draft visibility.}
    \label{fig:main_ablation}
    \vspace{-0.5em}
\end{figure}

\paragraph{Draft answer reliability estimation.}
We first examine whether the draft-answer reliability estimator $\kappa_{\text{a}}$ can effectively identify unreliable draft answers on GSM8K. Fig.~\ref{fig:main_ablation}~(left) compares \copt with a uniform selection under different thresholds $\tau_{\text{a}}$. For each threshold, we report the number of caught errors, the precision defined as the fraction of presumed errors that are truly erroneous, and the safe acceptance rate defined as the fraction of directly accepted answers that are correct. As $\tau_{\text{a}}$ becomes stricter, \copt selects fewer draft answers as unreliable, but the selected subset becomes increasingly concentrated on truly erroneous drafts. In contrast, uniform selection yields a much lower and nearly flat precision, indicating that uniform allocation of additional thinking fails to distinguish unreliable answers from reliable ones. These results confirm that $\kappa_{\text{a}}$ provides a meaningful reliability estimate.

\paragraph{Visibility control for draft answers.}
We next study the effect of the visibility estimator $\kappa_{\text{r}}$ during on-policy thinking on Math500. Fig.~\ref{fig:main_ablation}~(right) varies the threshold $\tau_{\text{r}}$ and reports the number of actual draft errors and the number of errors that are successfully corrected. Smaller $\tau_{\text{r}}$ corresponds to stricter visibility control, exposing the draft answer less often. As $\tau_{\text{r}}$ becomes stricter, \copt increases reasoning effort with less reliance on the potentially wrong draft answer, and corrects a larger fraction of erroneous draft answers. These results show that the estimator $\kappa_{\text{r}}$ plays a direct role in reflection and correction. If the draft answer is always exposed, unreliable content may continue to influence the on-policy thinking process and prevent the model from correcting its initial mistake. By contrast, dynamic visibility control allows \copt to hide the draft answer when the current thinking chunk appears unstable, while still allowing access to the draft when it provides useful partial information.

\section{Conclusion}

This paper introduces \copt, a training-free LLM reasoning pipeline that reverses the usual order of thinking and answering. \copt first elicits a draft answer and then invokes on-policy thinking conditioned on the draft when it appears unreliable. \copt recasts continuous embeddings, previously used for generation in latent \textsc{CoT} methods, as inference-time verifiers by contrasting the model's support for the same generated tokens under discrete-token and continuous-embedding inputs. Experiments on math, coding, and agentic reasoning tasks show that \copt consistently improves both accuracy and token efficiency, suggesting that on-policy thinking with contrastive verification from continuous embeddings provides a practical path toward more cost-effective LLM reasoning at inference time.

\newpage

\bibliographystyle{plainnat}
\bibliography{main_arxiv}

\newpage

\appendix
\section{Impact Statement}
\label{sec:impact}

This paper studies an inference-time method for reasoning LLMs. By improving accuracy and token efficiency, \copt may make reasoning LLMs more accessible and cost-effective for beneficial applications, such as scientific problem solving, programming assistance, and tool-using agents. At the same time, stronger and cheaper reasoning may also increase the usefulness of LLMs in harmful applications if deployed without proper safeguards. These risks are not specific to \copt, but follow from improving the capability of reasoning LLMs. \copt does not introduce a new class of societal risks beyond those associated with the underlying models.
\section{Usage of Large Language Models}
\label{sec:llm_usage}
We used LLMs and coding agents to assist with language polishing, improving readability, and debugging code. LLMs were not misused intentionally in any part of this work. All technical ideas and experimental results are the product of human efforts.
\section{Limitations}
\label{sec:limitations}

While \copt reformulates LLM reasoning in a training-free manner, its current design has a few limitations. First, the proposed KL estimators are evaluated on the realized trajectory available at inference time, rather than averaged over multiple continuations sampled from the model distribution. This matches our per-instance reasoning control setting, but may lead to higher variance than multi-sample estimates. We mitigate this by averaging over tokens, while future work may study more adaptive or lower-variance estimators. Second, \copt requires access to next-token probabilities. This is natural for open-weight models and inference systems that expose logits, but may be less convenient for closed APIs that only return text outputs. Future work may explore API-compatible variants that approximate the contrastive reliability using limited observable information.
\section{Supplementary Experiments}

\subsection{Ablation on Answer-Content Granularity}

\begin{table}[ht]
\centering
\caption{Ablation on the granularity of $\kappa_{\text a}$ estimation. The default setting computes $\kappa_{\text a}$ over the whole draft answer, while the answer-content setting computes it only on the extracted answer span.}
\vspace{0.8em}
\begingroup
\footnotesize
\setlength{\tabcolsep}{6.0pt}
\begin{tabular}{lcccc}
\toprule[1.2pt]
\multirow[c]{2}{*}{\raisebox{-0.5ex}{\textbf{Granularity}}}
& \multicolumn{2}{c}{\textbf{GSM8K}}
& \multicolumn{2}{c}{\textbf{Math500}} \\
\cmidrule(lr){2-3}
\cmidrule(lr){4-5}
& Acc.~(\%) & \# Tokens
& Acc.~(\%) & \# Tokens \\
\midrule[0.8pt]
\rowcolor{defaultBg}
Whole draft answer (default)
& 95.98 & 961
& 96.20 & 3609 \\
Answer content only
& \accup{\textbf{96.36}}{0.38}
& \tokdown{\textbf{885}}{-7.9}
& \accup{\textbf{96.40}}{0.20}
& \tokdown{\textbf{3214}}{-10.9} \\
\bottomrule[1.2pt]
\end{tabular}
\endgroup
\label{tab:granularity}
\end{table}

Tab.~\ref{tab:granularity} studies the effect of the granularity used for $\kappa_a$ estimation. In the default setting, we compute the KL-based reliability estimator over the whole draft answer. We find that a finer-grained variant can further improve performance when the final answer span is easy to identify. For example, in mathematics tasks, the final answer is usually enclosed by \texttt{\textbackslash boxed\{\}}, which allows us to extract the answer tokens and compute $\kappa_a$ only on this compact region. This setting improves GSM8K accuracy from 95.98\% to 96.36\% while reducing the average generation length from 961 to 885 tokens. On Math500, it improves accuracy from 96.20\% to 96.40\% and reduces the average generation length from 3609 to 3214 tokens.

The results suggest that more precise reliability estimation can be useful when a task provides a clear and compact answer region. However, such extraction is task-dependent and is not always available for reasoning tasks beyond mathematics. To keep the method general across domains, all results in this paper use the default setting, where $\kappa_a$ is computed over the whole draft.

\subsection{Ablation on the Maximum Draft-Answer Length}

\begin{table}[ht]
\centering
\caption{Ablation on the maximum draft-answer length.}
\vspace{0.8em}
\begingroup
\footnotesize
\setlength{\tabcolsep}{6.0pt}
\begin{tabular}{lcccc}
\toprule[1.2pt]
\multirow[c]{2}{*}{\raisebox{-0.5ex}{\textbf{Maximum Draft Length}}}
& \multicolumn{2}{c}{\textbf{GSM8K}}
& \multicolumn{2}{c}{\textbf{Math500}} \\
\cmidrule(lr){2-3}
\cmidrule(lr){4-5}
& Acc.~(\%) & \# Tokens
& Acc.~(\%) & \# Tokens \\
\midrule[0.8pt]
256
& \accdown{95.67}{0.31}
& \tokup{1075}{11.9}
& \accup{96.60}{0.40}
& \tokdown{3405}{-5.7} \\
512
& \accdown{95.37}{0.61}
& \tokup{1021}{6.2}
& \accup{\textbf{97.00}}{0.80}
& \tokup{3660}{1.4} \\
\rowcolor{defaultBg}
1,024 (default)
& \textbf{95.98} & \textbf{961}
& 96.20 & 3609 \\
2,048
& \accdown{95.52}{0.46}
& \tokup{965}{0.4}
& \accdown{94.40}{1.80}
& \tokdown{\textbf{3265}}{-9.5} \\
\bottomrule[1.2pt]
\end{tabular}
\endgroup
\label{tab:max_draft_length}
\end{table}

Tab.~\ref{tab:max_draft_length} studies the effect of the maximum generation length for draft answers. We set the maximum draft length to $1{,}024$ by default. This cap serves as a practical safeguard for the draft-first setting. We observe that for a small number of challenging tasks, reasoning LLMs may not be well calibrated to directly produce a concise answer before extended thinking. In such cases, the draft-answer stage can occasionally continue with repetitive or uninformative text. Bounding the draft length prevents these cases from consuming excessive tokens before the reliability estimator and subsequent on-policy thinking are applied.

Overall, the maximum draft-answer length does not behave as a highly sensitive hyperparameter. On GSM8K, the default length of $1{,}024$ achieves the best accuracy while using the fewest tokens among all settings. On Math500, shorter limits such as $256$ and $512$ can improve accuracy, suggesting that task-specific draft caps may provide additional gains. However, increasing the cap to \(2{,}048\) does not improve accuracy on either benchmark. These results suggest that the draft-length cap mainly serves to prevent unusually long or repetitive drafts, rather than acting as a sensitive hyperparameter. We therefore use $1{,}024$ as a simple default for all experiments, which provides stable performance without task-specific tuning.

The only exception is the multi-turn ZebraArena benchmark, where we use a smaller maximum draft length of $512$. Since ZebraArena requires multiple rounds of interaction, the effective draft budget can accumulate across turns. A shorter per-turn cap helps control total context growth while still allowing the model to produce a concise draft answer at each turn. 

\subsection{Detailed Results on the LeetCode-Contest Benchmark}

\begin{table}[ht]
\centering
\caption{Per-split accuracy and generation length comparison on the LeetCode-Contest benchmark.}
\vspace{0.8em}
\begingroup
\footnotesize
\setlength{\tabcolsep}{5.0pt}
\begin{tabular}{lccccc}
\toprule[1.2pt]
\textbf{Method}
& \textbf{Easy}
& \textbf{Medium}
& \textbf{Hard}
& \textbf{Overall}
& \textbf{\# Tokens} \\
\midrule[0.8pt]
CoT
& 57.78
& 68.13
& 43.18
& 59.44
& 7306 \\
CoT~(Greedy)
& 64.44
& 58.24
& 47.73
& 57.22
& 6975 \\
\midrule[0.5pt]
\multirow{2}{*}{\textbf{CopT~(Ours)}}
& \cellcolor{coptMainBg}\accup{64.44}{6.66}
& \cellcolor{coptMainBg}\accup{72.53}{4.40}
& \cellcolor{coptMainBg}\accup{54.55}{11.37}
& \cellcolor{coptMainBg}\accup{66.11}{6.67}
& \cellcolor{coptMainBg}\tokup{7607}{4.1} \\
& \accdown{51.11}{6.67}
& \accup{69.23}{1.10}
& \accup{54.55}{11.37}
& \accup{61.11}{1.67}
& \tokdown{6993}{-4.3} \\
\bottomrule[1.2pt]
\end{tabular}
\endgroup
\label{tab:leetcode_detailed}
\end{table}

Tab.~\ref{tab:leetcode_detailed} reports per-split results on the LeetCode-Contest benchmark. \copt improves the overall accuracy from $59.44\%$ to $66.11\%$ with a similar token budget. The largest gain appears on the hard split, where \copt improves accuracy from $43.18\%$ to $54.55\%$, yielding an absolute gain of $+11.37\%$. This suggests that on-policy thinking is especially useful on harder problems, where the initial draft answer is more likely to require further reflection and correction.

The two \copt rows show different accuracy-efficiency trade-offs. The set with increasing reasoning effort achieves the best overall accuracy, improving all three splits over standard \textsc{CoT}, with only a small increase in generation length. The other set reduces token usage by $4.3\%$ while still improving overall accuracy by $+1.67\%$. Although the other set reduces accuracy on the easy split, it preserves the gain on the hard split and still improves the medium split. These results show that \copt can trade generation length for accuracy in a controlled way.

\subsection{Detailed Results on BFCL v4 Benchmark}

\begin{table*}[ht]
\centering
\caption{Per-split accuracy comparison on the BFCL v4 benchmark.}
\label{tab:bfcl_split}
\vspace{0.8em}
\begingroup
\footnotesize
\setlength{\tabcolsep}{1pt}

\begin{tabular}{llcccccc}
\toprule[1.2pt]
\textbf{Model} & \textbf{Method}
& \textbf{live\_multiple}
& \textbf{live\_parallel}
& \textbf{live\_parallel\_multiple}
& \textbf{live\_simple}
& \textbf{multiple}
& \textbf{parallel} \\
\midrule[0.8pt]
\multirow[c]{3}{*}{Qwen3.5-2B}
& CoT
& 73.41
& 62.50
& \textbf{54.17}
& \textbf{72.87}
& 90.50
& 81.50 \\
& \multirow[c]{2}{*}{\textbf{CopT~(Ours)}}
& \textbf{74.93}
& \textbf{81.25}
& \textbf{54.17}
& 69.77
& \textbf{91.50}
& \textbf{82.00} \\
&
& {\scriptsize\textcolor{oliveGreen}{(+1.52)}}
& {\scriptsize\textcolor{oliveGreen}{(+18.75)}}
& {\scriptsize\textcolor{oliveGreen}{(+0.00)}}
& {\scriptsize\textcolor{blue!70!black}{(-3.10)}}
& {\scriptsize\textcolor{oliveGreen}{(+1.00)}}
& {\scriptsize\textcolor{oliveGreen}{(+0.50)}} \\
\midrule[0.5pt]
\multirow[c]{3}{*}{Qwen3.5-35B-A3B}
& CoT
& 81.20
& \textbf{87.50}
& \textbf{79.17}
& 85.27
& 93.50
& \textbf{91.50} \\
& \multirow[c]{2}{*}{\textbf{CopT~(Ours)}}
& \textbf{81.58}
& \textbf{87.50}
& \textbf{79.17}
& \textbf{85.66}
& \textbf{95.00}
& \textbf{91.50} \\
&
& {\scriptsize\textcolor{oliveGreen}{(+0.38)}}
& {\scriptsize\textcolor{oliveGreen}{(+0.00)}}
& {\scriptsize\textcolor{oliveGreen}{(+0.00)}}
& {\scriptsize\textcolor{oliveGreen}{(+0.39)}}
& {\scriptsize\textcolor{oliveGreen}{(+1.50)}}
& {\scriptsize\textcolor{oliveGreen}{(+0.00)}} \\
\bottomrule[1.2pt]
\end{tabular}

\vspace{0.9em}

\begin{tabular}{llcccccc}
\toprule[1.2pt]
\textbf{Model} & \textbf{Method}
& \textbf{parallel\_multiple}
& \textbf{simple\_java}
& \textbf{simple\_javascript}
& \textbf{simple\_python}
& \textbf{Overall}
& \textbf{\# Tokens} \\
\midrule[0.8pt]
\multirow[c]{3}{*}{Qwen3.5-2B}
& CoT
& 79.00
& 69.00
& \textbf{60.00}
& 88.50
& 77.53
& 234 \\
& \multirow[c]{2}{*}{\textbf{CopT~(Ours)}}
& \textbf{79.50}
& \textbf{72.00}
& \textbf{60.00}
& \textbf{89.25}
& \textbf{78.37}
& \textbf{164} \\
&
& {\scriptsize\textcolor{oliveGreen}{(+0.50)}}
& {\scriptsize\textcolor{oliveGreen}{(+3.00)}}
& {\scriptsize\textcolor{oliveGreen}{(+0.00)}}
& {\scriptsize\textcolor{oliveGreen}{(+0.75)}}
& {\scriptsize\textcolor{oliveGreen}{(+0.84)}}
& {\scriptsize\textcolor{blue!70!black}{(-29.9\%)}} \\
\midrule[0.5pt]
\multirow[c]{3}{*}{Qwen3.5-35B-A3B}
& CoT
& \textbf{88.00}
& 78.00
& \textbf{78.00}
& 93.50
& 85.77
& 235 \\
& \multirow[c]{2}{*}{\textbf{CopT~(Ours)}}
& \textbf{88.00}
& \textbf{82.00}
& 72.00
& \textbf{95.50}
& \textbf{86.45}
& \textbf{168} \\
&
& {\scriptsize\textcolor{oliveGreen}{(+0.00)}}
& {\scriptsize\textcolor{oliveGreen}{(+4.00)}}
& {\scriptsize\textcolor{blue!70!black}{(-6.00)}}
& {\scriptsize\textcolor{oliveGreen}{(+2.00)}}
& {\scriptsize\textcolor{oliveGreen}{(+0.68)}}
& {\scriptsize\textcolor{blue!70!black}{(-28.5\%)}} \\
\bottomrule[1.2pt]
\end{tabular}

\endgroup
\label{tab:bfcl_detailed}
\end{table*}

Tab.~\ref{tab:bfcl_detailed} reports per-split results on the non-live and live splits of BFCL v4. The results show that \copt preserves or improves accuracy in most cases across different function-calling categories while reducing unnecessary reasoning tokens, suggesting that \copt provides a better accuracy-efficiency trade-off for structured agentic reasoning.
\section{Supplementary Details}
\label{sec:supplementary_details}

\subsection{Implementation Details}

All experiments are conducted on a single NVIDIA H200 GPU. We use the default generation settings of Qwen3~\citep{yang2025qwen3} and Qwen3.5~\citep{qwen3.5} models, including temperature of 0.6, top-p of 0.95, top-k of 20, and min-p of 0 for all experiments. The chunk size of \copt is set to $C = \max\!\left(1, \left\lfloor \frac{T_a}{4} \right\rfloor \right)$, which balances timely estimation with low additional cost. By default, the draft answer is invisible to the first chunk of on-policy thinking, since the earliest value of $\kappa_{\text{r}}$ is available only after the first chunk is completed. We set the maximum draft answer length to 512 for the multi-turn ZebraArena benchmark and 1,024 for all other benchmarks.

\subsection{Benchmark Details}
\label{supp_benchmark}

We evaluate \copt on 10 reasoning benchmarks, covering GSM8K~\citep{cobbe2021training}, Math500~\citep{hendrycks2021measuring}, AIME 2024~\citep{hf_aime24_dataset}, AIME 2025~\citep{hf_aime25_dataset}, GPQA Diamond~\citep{rein2024gpqa} for mathematics and STEM reasoning; HumanEval~\citep{chen2021evaluating}, LeetCode-Contest~\citep{guo2024deepseekcoder}, MBPP~\citep{austin2021program} for coding reasoning; and BFCL v4~\citep{patil2025bfcl}, ZebraArena~\citep{zhao2026zebraarena} for agentic reasoning.

\begin{itemize}
    \item \textbf{GSM8K}: We evaluate on the test set of 1,319 grade-school math problems. The benchmark tests whether a model can solve natural-language arithmetic questions that often require several reasoning steps. \hf: \url{https://huggingface.co/datasets/openai/gsm8k}.
    \item \textbf{Math500}: We use the curated set of 500 problems from the MATH benchmark. The problems span multiple high-school competition mathematics areas, including algebra, geometry, number theory, precalculus, and intermediate algebra. \hf: \url{https://huggingface.co/datasets/HuggingFaceH4/MATH-500}.
    \item \textbf{AIME 2024}: The benchmark contains 30 problems from the 2024 American Invitational Mathematics Examination, covering both AIME I and AIME II. Each problem requires a concise numeric answer and is designed to test competition-level mathematical reasoning. \hf: \url{https://huggingface.co/datasets/HuggingFaceH4/aime_2024}.
    \item \textbf{AIME 2025}: The benchmark contains 30 problems from the 2025 American Invitational Mathematics Examination, covering both AIME I and AIME II. Each problem requires a concise numeric answer and continues the focus on competition-style math reasoning with challenging questions that test symbolic and logical skills. \hf: \url{https://huggingface.co/datasets/yentinglin/aime_2025}.
    \item \textbf{GPQA Diamond}: We use the Diamond subset of GPQA, which consists of 198 expert-verified multiple-choice STEM questions. The questions mainly cover mathematics, physics, chemistry, biology, and computer science, and are intended to be difficult for non-experts. The problems are designed to evaluate expert-level factual knowledge and reasoning ability. \hf: \url{https://huggingface.co/datasets/hendrydong/gpqa_diamond_mc}.
    \item \textbf{HumanEval}: We use the 164 hand-written Python programming tasks from HumanEval. Each task provides a function signature and docstring, and correctness is measured by executing the generated function against unit tests.  \hf: \url{https://huggingface.co/datasets/openai/openai_humaneval}.
    \item \textbf{LeetCode-Contest}: We evaluate on 180 programming contest problems collected from LeetCode contests. The benchmark contains problems of different difficulty levels, and model outputs are judged by whether the generated solutions pass all the associated tests. \hf: \url{https://huggingface.co/datasets/TechxGenus/LeetCode-Contest}.
    \item \textbf{MBPP}: We use the sanitized test split, which contains 257 Python programming problems. Each example includes a natural-language task description, a reference solution, and unit tests used for execution-based scoring.  \hf: \url{https://huggingface.co/datasets/google-research-datasets/mbpp}.
    \item \textbf{BFCL v4}: The Berkeley Function Calling Leaderboard v4 evaluates the ability of LLMs to invoke functions and tools accurately in realistic agentic settings. We use all subtasks from the non-live and live splits, which contain 2,501 problems in total. \github: \url{https://github.com/shishirpatil/gorilla}.
    \item \textbf{ZebraArena}: A diagnostic simulation environment for evaluating multi-turn agentic reasoning in tool-augmented LLMs. It is built on Zebra logic puzzles under a missing-clues setting, where models must query the environment for hidden facts or relations and then solve a uniquely verifiable constraint satisfaction problem. \github: \url{https://github.com/wanjiaZhao1203/ZebraArena}.
\end{itemize}

Following the evaluation setup of Qwen3 and Qwen3.5, we allocate a large maximum output budget to allow sufficient reasoning. Specifically, we set the maximum generation length to 32,768 tokens for GSM8K, Math500, GPQA Diamond, HumanEval, LeetCode-Contest, MBPP, and BFCL v4, and 38,912 tokens for AIME 2024 and AIME 2025. For ZebraArena, we set the maximum generation length to 32,768 tokens for the small split, 65,536 tokens for the medium split, and 98,304 tokens for the large split. We repeat evaluations eight times and report the average accuracy for both \copt and baselines on the AIME 2024 and AIME 2025 benchmarks.

\subsection{Sequence Distributions for the KL Estimators}

We define the sequence distributions used by the KL estimators for fixed generated lengths.

\paragraph{For $\kappa_{\text{a}}$ in the draft answer stage.}

For a fixed draft length \(T_a\), the discrete-prefix continuation distribution over the draft answer is
\[
p_\theta(a_{1:T_a}\mid q)
:=
\prod_{t=1}^{T_a}
p_\theta(a_t \mid q, a_{<t}).
\]
The corresponding continuous-prefix continuation distribution is
\[
p_\theta^{e}(a_{1:T_a}\mid q)
:=
\prod_{t=1}^{T_a}
p_\theta(a_t \mid q, e_{<t}),
\]
where the discrete prefix preceding each draft token is replaced by the cached continuous embeddings.

\paragraph{For $\kappa_{\text{r}}$ in the on-policy thinking stage.}

We partition the on-policy thinking trajectory into chunks of length $C$. Let the $k$-th chunk start at position
\[
s_k := T_a + 1 + (k-1)C,
\]
and span positions $[s_k, s_k+C-1]$. Let $m_k$ denote the visibility state of the draft answer for the $k$-th chunk.

Conditioned on the visibility-controlled draft-answer input $a^{(m_k)}$ and the prefix $r_{T_a+1:s_k-1}$ preceding the current chunk, the student chunk-level continuation distribution is
\[
p_\theta(r_{s_k:s_k+C-1}\mid q, a^{(m_k)}, r_{T_a+1:s_k-1})
:=
\prod_{t=s_k}^{s_k+C-1}
p_\theta(r_t \mid q, a^{(m_k)}, r_{T_a+1:t-1}).
\]
The corresponding continuous-prefix intra-chunk continuation distribution is
\[
p_\theta^{e}(r_{s_k:s_k+C-1}\mid q, a^{(m_k)}, r_{T_a+1:s_k-1})
:=
\prod_{t=s_k}^{s_k+C-1}
p_\theta(r_t \mid q, a^{(m_k)}, r_{T_a+1:s_k-1}; e_{s_k:t-1}),
\]
where the already generated content inside the current chunk is replaced by the cached continuous embeddings $e_{s_k:t-1}$, while the prefix before the current chunk remains discrete.

\section{Additional Derivations for Answer-Relevant Uncertainty}
\label{app:answer_relevant_uncertainty}

In this section, we provide the full derivation and several consequences of \Cref{thm:answer_relevant_uncertainty_main}. The main message is that the reverse-KL estimator measures uncertainty that affects the next answer token, rather than uncertainty over latent reasoning states themselves.

\begin{proof}[Proof of \Cref{thm:answer_relevant_uncertainty_main}]

Recall the mixture-linear prefix model under \Cref{assump:mixture_soft_prefix}. The local reverse-KL contribution is
\[
\kappa(S,A)
=
\log P_S(A)-\log \bar P_w(A).
\]

Taking expectation gives
\[
\begin{aligned}
\mathbb{E}[\kappa(S,A)]
&=
\sum_{s\in\mathcal{S}}w(s)
\sum_a P_s(a)
\left[
\log P_s(a)-\log \bar P_w(a)
\right] \\
&=
\sum_{s\in\mathcal{S}}w(s)
D_{\mathrm{KL}}
\!\left(
P_s
\,\middle\|\,
\bar P_w
\right).
\end{aligned}
\]
Moreover, under the joint distribution
\[
\Pr(S=s,A=a)=w(s)P_s(a),
\]
the marginal distribution of $A$ is
\[
\Pr(A=a)
=
\sum_s w(s)P_s(a)
=
\bar P_w(a).
\]
Therefore,
\[
\begin{aligned}
I(S;A)
&=
\sum_{s,a}
w(s)P_s(a)
\log
\frac{\Pr(S=s,A=a)}
{\Pr(S=s)\Pr(A=a)}
\\
&=
\sum_{s,a}
w(s)P_s(a)
\log
\frac{P_s(a)}
{\bar P_w(a)}
\\
&=
\sum_s w(s)
D_{\mathrm{KL}}
\!\left(
P_s
\,\middle\|\,
\bar P_w
\right).
\end{aligned}
\]
Hence,
\[
\mathbb{E}[\kappa(S,A)]
=
I(S;A),
\]
which completes the proof.

\end{proof}

\paragraph{Harmless latent-state uncertainty.}

If all latent states in the support of $w$ induce the same next-token distribution,
\[
P_s=P_\star
\qquad
\text{for all }s\in\mathrm{supp}(w),
\]
then
\[
\bar P_w
=
\sum_s w(s)P_s
=
P_\star.
\]
Therefore,
\[
\kappa(S,A)
=
\log P_\star(A)-\log P_\star(A)
=
0
\]
almost surely. Thus, the reverse-KL score is zero even if the latent-state entropy $H(S)$ is large.

This formalizes the intuition that a continuous prefix may encode a superposition of many states and still be reliable. If all plausible states agree on the next-token distribution, then the uncertainty is irrelevant to the answer.

\paragraph{Stability under approximately equivalent states.}

More generally, suppose all state-conditioned distributions are close to a common distribution $P_\star$. Then
\[
\mathbb{E}[\kappa(S,A)]
\le
\sum_s w(s)
D_{\mathrm{KL}}
\!\left(
P_s
\,\middle\|\,
P_\star
\right).
\]

To see this, note that
\[
\begin{aligned}
\sum_s w(s)
D_{\mathrm{KL}}(P_s\|P_\star)
&=
\sum_s w(s)
\sum_a P_s(a)
\log\frac{P_s(a)}{P_\star(a)} \\
&=
\sum_s w(s)
\sum_a P_s(a)
\left[
\log\frac{P_s(a)}{\bar P_w(a)}
+
\log\frac{\bar P_w(a)}{P_\star(a)}
\right] \\
&=
\sum_s w(s)
\sum_a P_s(a)
\log\frac{P_s(a)}{\bar P_w(a)}
+
\sum_s w(s)
\sum_a P_s(a)
\log\frac{\bar P_w(a)}{P_\star(a)} \\
&=
\sum_s w(s)
D_{\mathrm{KL}}(P_s\|\bar P_w)
+
\sum_a
\left[
\sum_s w(s)P_s(a)
\right]
\log\frac{\bar P_w(a)}{P_\star(a)} \\
&=
\sum_s w(s)
D_{\mathrm{KL}}(P_s\|\bar P_w)
+
\sum_a
\bar P_w(a)
\log\frac{\bar P_w(a)}{P_\star(a)} \\
&=
\sum_s w(s)
D_{\mathrm{KL}}(P_s\|\bar P_w)
+
D_{\mathrm{KL}}(\bar P_w\|P_\star).
\end{aligned}
\]

Since KL divergence is nonnegative,
\[
\sum_s w(s)
D_{\mathrm{KL}}
\!\left(
P_s
\,\middle\|\,
\bar P_w
\right)
\le
\sum_s w(s)
D_{\mathrm{KL}}
\!\left(
P_s
\,\middle\|\,
P_\star
\right).
\]
By Theorem~\ref{thm:answer_relevant_uncertainty_main},
\[
\mathbb{E}[\kappa(S,A)]
=
\sum_s w(s)
D_{\mathrm{KL}}
\!\left(
P_s
\,\middle\|\,
\bar P_w
\right),
\]
which proves the claim.

\paragraph{Deterministic-token case.}

Suppose each latent state $s$ deterministically implies one answer token
\[
g(s)\in\mathcal{A},
\]
so that
\[
P_s(a)
=
\mathbbm{1}\{a=g(s)\}.
\]
Define the marginal probability of the answer token induced by the soft prefix state as
\[
\rho(a)
=
\Pr(g(S)=a)
=
\sum_{s:g(s)=a}w(s).
\]
Then
\[
\bar P_w(a)
=
\sum_s w(s)\mathbbm{1}\{a=g(s)\}
=
\rho(a).
\]
Since $A=g(S)$ deterministically,
\[
\kappa(S,A)
=
\log 1-\log \rho(g(S))
=
-\log \rho(g(S)).
\]
Taking expectation gives
\[
\mathbb{E}[\kappa(S,A)]
=
\sum_a \rho(a)(-\log \rho(a))
=
H(g(S)).
\]

Therefore, under the deterministic-answer setting, the reverse-KL score is exactly the entropy of the induced answer token $g(S)$ instead of the entropy of the latent state $S$.


\end{document}